\documentclass{article}
\usepackage{ijcai26}
\usepackage{times}
\usepackage{soul}
\usepackage{url}
\usepackage[hidelinks]{hyperref}
\usepackage[utf8]{inputenc}
\usepackage[small]{caption}
\usepackage{graphicx}
\usepackage{amsmath}
\usepackage{amsthm}
\usepackage{booktabs}
\usepackage{algorithm}
\usepackage{algorithmic}
\usepackage[switch]{lineno}
\usepackage{xspace}
\usepackage{thmtools}
\usepackage[inline]{enumitem}
\usepackage{soul}
\usepackage{todonotes}
\usepackage{ifthen}

\newboolean{shortver}
\setboolean{shortver}{false} 

\DeclareRobustCommand{\thmcite}[1]{\cite{#1}}

\newtheorem{theorem}{Theorem}
\newtheorem{lemma}[theorem]{Lemma}
\newtheorem{proposition}[theorem]{Proposition}
\newtheorem{corollary}[theorem]{Corollary}
\theoremstyle{definition}
\newtheorem{definition}[theorem]{Definition}

\newtheorem*{theoremrestated}{Theorem}
\newtheorem*{propositionrestated}{Proposition}

\newcommand{\band}{\bigwedge}
\newcommand{\all}{\forall}
\newcommand{\limp}{\supset}
\newcommand{\FOtwo}{\ensuremath{\textsc{FO}^2}\xspace}
\newcommand{\UTC}{\ensuremath{\textsc{UTC}}\xspace}
\newcommand{\ALCOU}{{\cal ALCO}(U)}
\newcommand{\ALCQOU}{{\cal ALCQO}(U)}
\newcommand{\lang}{{\cal L}\xspace}
\newcommand{\M}{{\cal M}}
\newcommand{\union}{\cup}
\newcommand{\properplus}{\textrm{proper}^+}
\newcommand{\D}{\ensuremath{\mathcal{D}}}
\newcommand{\forget}{\mathit{forget}}
\newcommand{\true}{{\mbox{\footnotesize {\sc true}}}}
\newcommand{\false}{\mbox{\footnotesize {\sc false}}}
\newcommand{\sdo}{\mathit{do}}
\newcommand{\Poss}{\mathit{Poss}}
\newcommand{\LEF}{\mathrm{LE}}
\newcommand{\NLEF}{\mathrm{NLE}}
\newcommand{\size}[1]{|#1|}
\newcommand{\consize}[2]{|#2|_{#1}}
\newcommand{\NWSC}{\mathit{NWS}}
\newcommand{\SNC}{\mathit{SNC}}

\title{On the Size Complexity and Decidability of First-Order Progression%
\ifthenelse{\boolean{shortver}}%
{\footnote{An extended version of this paper with proofs is available at arXiv:YYMM.NNNNN.}}{\footnote{Extended version of identically-titled paper at IJCAI 2026.}}
}

\author{
Jens Cla{\ss}en$^{1}$\thanks{Corresponding Author.}
\and
Daxin Liu$^{2}$\footnotemark[2]
\\
\affiliations
$^1$Department of People and Technology, Roskilde University, Denmark\\
$^2$State Key Laboratory for Novel Software Technology, Nanjing University, China\\
\emails
classen@ruc.dk,
daxin.liu@nju.edu.cn
}

\begin{document}

\maketitle

\begin{abstract}
Progression, the task of updating a knowledge base to reflect action effects, generally requires second-order logic. Identifying first-order special cases, by restricting either the knowledge base or action effects, has long been a central topic in reasoning about actions. It is known that local-effect, normal, and acyclic actions, three increasingly expressive classes, admit first-order progression. However, a systematic analysis of the size of such progressions, crucial for practical applications, has been missing. In this paper, using the framework of Situation Calculus, we show that under reasonable assumptions, first-order progression for these action classes grows only polynomially. Moreover, we show that when the KB belongs to decidable fragments such as two-variable first-order logic or universal theories with constants, the progression remains within the same fragment, ensuring decidability and practical applicability.
\end{abstract}

\section{Introduction}

Our motivation is derived from the aim of enabling intelligent agents to act in real-world scenarios, typically involving incomplete information about the state of the world and unbounded domains of objects.
For this purpose, we use the framework of Situation Calculus \cite{McCarthy/Hayes:1969,Reiter2001}, a widely studied first-order (FO) formalism for reasoning about action and change.
A central problem for various reasoning tasks such as planning, verification, and synthesis, is that of {\em projection}, i.e., to determine if a query formula $\phi$ will hold after the execution of an action $\alpha$, given a theory $\D$ consisting of a knowledge base (KB) describing the current state of affairs and axioms encoding the effects of actions.
The two most common approaches are those of {\em regressing} $\phi$ through $\alpha$ and then testing the result against the original KB, and that of {\em progressing} the KB through $\alpha$ and then testing the original $\phi$ against this updated KB.
Arguably, regression is less suited for practical applications because it needs to be performed for every query formula anew, while it is sufficient to determine a progression once for each action.

We thus focus our attention on progression in this paper.
Unfortunately, there is a strong negative result due to Lin and Reiter \shortcite{DBLP:journals/ai/LinR97}, which says that representing the progression of an FO theory requires second-order (SO) logic in general.
Identifying FO special cases, by restricting the KB or action effects, has therefore long been a central topic in reasoning about actions and change.
Three increasingly expressive\footnote{The inclusion is not exact, as we discuss in Sec.~\ref{sec:SpaceComplexity}.} such classes are the local-effect, normal \cite{DBLP:conf/ijcai/LiuL09}, and acyclic \cite{DBLP:conf/ijcai/0002C24} ones, all of which were shown to be effectively FO progressable.

However, for practical feasibility, at least two additional aspects are crucial that we address in this paper.
First, a formal analysis of the size complexity of such progressions has been missing.
After formal preliminaries in Sec.~\ref{sec:Preliminaries}, we show in Sec.~\ref{sec:SpaceComplexity} that under certain assumptions, the FO progression for the aforementioned classes grows only polynomially, or even linearly. 
Besides, we want to guarantee that queries can be effectively evaluated against the progression result, for instance by resorting to decidable fragments of first-order logic (FOL) such as the two-variable fragment \FOtwo \cite{DBLP:journals/bsl/GradelKV97} or universal theories with constants \UTC \cite{DBLP:conf/kr/Arenas18}.
We hence show these fragments to be closed under progression in case of \FOtwo (Sec.~\ref{sec:FOTwo}) and \UTC (Sec.~\ref{sec:UTC}). We discuss related work and conclude in Sec.~\ref{sec:RelatedWork}.

\section{Preliminaries}
\label{sec:Preliminaries}

We assume familiarity with FOL including equality, here denoted by $\lang$. We will use ``dot'' notation to indicate that the quantifier preceding the dot has maximum scope, e.g., $\all x. \phi(x) \supset \psi(x)$ stands for $\all x (\phi(x) \supset \psi(x))$.  We omit leading universal quantifiers and assume free variables are implicitly $\forall$-quantified,  e.g., we identify $\phi(x)$ with $\all x. \phi(x)$.  We use
$\psi \Leftrightarrow \psi$ to mean $\phi$ and $\psi$ are logically
equivalent. We denote by $\phi(\mu/\mu')$ the formula obtained by simultaneously substituting every occurrence of expression (term or formula) $\mu$  in $\phi$ with $\mu'$.
Let $\lang^2$ denote the SO extension of $\lang$.

\subsection{Situation Calculus}

The Situation Calculus \cite{Reiter2001} $\lang_{sc}$ is a three-sorted
language with some SO features suitable for describing
dynamic worlds. The sorts are used to distinguish \emph{action},
\emph{situation}, and \emph{object}. $\lang_{sc}$ uses a distinct constant $S_0$ to denote the initial situation; a binary function $\sdo(a,s)$ to generate the successor
situation of situation $s$ from doing action $a$; a binary relation
$\Poss(a,s)$ to express action $a$ being executable in situation $s$. We use $a$ and $s$ for variables of sort action and situation, respectively. 
Given a ground action $\alpha$, we use $S_\alpha$ to refer
to the situation $do(\alpha,S_0)$. \emph{Fluents} are predicates whose last argument is a situation term; we omit function fluents for simplicity.
A (possibly SO) formula $\phi$ is {\em uniform} in a situation
term $s$ if $\phi$ mentions no situation terms except
$s$, quantifies no situation variables, and mentions no  
$\mathit{Poss}$.

\begin{definition}[Basic Action Theory]
  \label{def:BAT}
A \emph{basic action theory}
(BAT) in $\lang_{sc}$ is a set of axioms
\[\D=\Sigma_{ind} \cup \D_{ap} \cup \D_{ss} \cup \D_{una} \cup \D_{S_0}, \textrm{where}\]
\begin{enumerate}
\item $\Sigma_{ind}$ is a set of domain-independent axioms \cite{DBLP:journals/etai/LevesquePR98} that ensure
  situations are well-structured; 
\item $\D_{ap}$ is a set of action precondition axioms;
\item $\D_{ss}$ is a set of successor state axioms (SSAs), one for
  each fluent predicate $F$, of the form
  \[F(\vec{x},\sdo(a,s)) \equiv \gamma^+_{F}(\vec{x},a,s) \vee \neg
    \gamma^-_{F}(\vec{x},a,s) \land F(\vec{x},s),\] where
  $\gamma^+_{F}$ and $\gamma^-_{F}$, the conditions under which action $a$
  causes $F$ to become true and false, respectively, have to be uniform in $s$ and satisfy
  $\D \models \neg (\gamma^+_F \land \gamma^-_F)$;
\item $\D_{una}$ is the set of unique names axioms for actions:
  $A(\vec{x}) \neq A'(\vec{y})$, and
  $A(\vec{x}) = A(\vec{y}) \supset \vec{x} = \vec{y}$, where $A,A'$
  are distinct action symbols;
\item $\D_{S_0}$, the initial database (or initial KB), is a finite
  set of sentences uniform in $S_0$.
\end{enumerate}
\end{definition}

\subsection{Progression}

A progression should retain all proper information, i.e., logical
entailment in terms of the future of the initial KB. Lin and Reiter \shortcite{DBLP:journals/ai/LinR97} provide a model-theoretical definition of progression and prove that progression is SO definable. Vassos and Levesque \shortcite{DBLP:journals/ai/VassosL13} prove that any correct progression is equivalent to the SO representation of \cite{DBLP:journals/ai/LinR97}. Here we use the definition from \cite{DBLP:journals/ai/VassosL13}.

Let $\D_{ss}[\alpha,S_0]$ be the instantiation of
$\D_{ss}$ wrt $\alpha$ and $S_0$, i.e. $\D_{ss}[\alpha,S_0] \Leftrightarrow
\band_{F} \forall \vec{x}. F(\vec{x},\sdo(\alpha,S_0)) \equiv \Phi_F(\vec{x},\alpha,S_0)$, where
$\Phi_F$ denotes the RHS of the SSA for fluent $F$. Let
$F_1,\ldots F_n$ be the set of all fluents. For each fluent $F_i$, we
introduce a new predicate symbol $P_i$. Let $\phi \uparrow S_0$ be the result of replacing every $F_i(\vec{t},S_0)$ in $\phi$ by
$P_i(\vec{t})$ and call $P_i$ the {\em lifting predicate} for $F_i$. For a
finite set of formulas $\Sigma$, we also use $\Sigma$ to denote the
conjunctions of its elements. 
\begin{definition}
\label{def:progression}
     Let $\D$ be a basic action theory, $\alpha$ an executable  ground action (under $\D_{ap}$), 
and
  $\D_{S_\alpha}$ a set of (FO or SO) sentences
  uniform in $S_\alpha$. We say that $\D_{S_\alpha}$ is a
  \emph{progression} of $\D_{S_0}$ wrt $\alpha, \D$ iff $\D_{una} \cup \D_{S_\alpha}  \Leftrightarrow \D_{una}~\land$
\begin{equation}
\label{eq:linreiter}
\exists \vec{R}. \{(\D_{S_0} \union \D_{ss}[\alpha,S_0])\uparrow S_0\}(\vec{P} / \vec{R})
\end{equation}
where $\vec{R}=\{R_1,\ldots,R_n\}$ are SO predicate variables.
\end{definition}

Def.~\ref{def:progression} is correct in the sense that for any $\phi$ uniform in $S_\alpha$, if $\D_{S_\alpha}$ is a progression of $\D_{S_0}$ wrt $\alpha, \D$, then $\D \models \phi \textrm{ iff } (\D-\D_{S_0})\union \D_{S_\alpha} \models \phi$ \cite[Thm.~1]{DBLP:journals/ai/VassosL13}.

\subsection{Forgetting}
\label{subsec:FirstOrderProgressableActionTheories}

There are two approaches to identifying fragments of the Situation Calculus where progression is FO definable: one that restricts the structure of the initial theory, and one that restricts action effects. We focus on the latter in this paper, in particular, we consider \emph{local-effect action theories} (BAT-LE), \emph{normal action theories} (BAT-NR), and \emph{acyclic action theories} (BAT-AC). The next sections give definitions for these classes, presents the associated progression methods, and provides size complexity results as new contribution.

The progression of a theory through an action can be interpreted in terms of {\em forgetting} all information about the original state of the world, and only keeping information relevant to the new state.
Intuitively, forgetting a ground atom (or predicate) in a theory leads to a weaker theory that entails the same set of sentences that are ``irrelevant'' to the atom (or predicate). For a sentence $\phi$ and ground atom $P(\Vec{t})$, let $\phi[P(\Vec{t})]$ be the formula obtained by replacing every occurrence of the form $P(\Vec{t}')$ in $\phi$ with $[\Vec{t}=\vec{t}' \land P(\vec{t})] \vee [\Vec{t}\neq \vec{t}' \land P(\vec{t}')]$. Clearly, $\phi \Leftrightarrow \phi[P(\Vec{t})]$. Let $\phi^{P(\Vec{t})}_+$ and $\phi^{P(\Vec{t})}_-$ be formulas obtained by replacing $P(\Vec{t})$ in  $\phi[P(\Vec{t})]$ with $\true$ and $\false$, respectively.
\begin{theorem}[\thmcite{lin1994forget}]
\label{thm:forgetting}
Let $P(\Vec{t})$ be a ground atom, $P$ a predicate symbol, and $\phi$ a sentence. Then
    \begin{itemize}
        \item $\forget(\phi, P(\Vec{t})) \Leftrightarrow \phi^{P(\Vec{t})}_+ \vee \phi^{P(\Vec{t})}_-$;
        \item $\forget(\phi, P) \Leftrightarrow \exists R.\phi(P/R)$, where $R$ is a SO variable.
    \end{itemize}
\end{theorem}
\noindent Consequently, progression in Eq.~\eqref{eq:linreiter} amounts to adding the effects of the action ($\D_{ss}[\alpha,S_0]$) and forgetting the truth of fluents at situation $S_0$ ($\exists$-quantified from the outside).
Thus, progression is FO definable in cases where the SO quantifiers in \eqref{eq:linreiter} can be eliminated.

W.l.o.g., we assume $\gamma^{\pm}_F(\vec{x},a,s)$ in SSAs are disjunctions of formulas of the form (one for each action symbol $A$)
$$ \exists \vec{z}. \bigl( a = A(\vec{v}) \land \phi^\pm_A(\vec{x},\vec{z},s) \bigr)$$
where $A(\vec{v})$ is an action term, $\vec{v}$ may contain
      some variables from $\vec{x}$, and the remaining variables are
      $\vec{z}$; $\phi^{\pm}_A(\vec{x},\vec{z},s)$ is a formula uniform in $s$ with no action terms, and where all free variables symbols are among $\vec{x}$, $\vec{z}$, and $s$. Under this assumption and $\D_{una}$,       
      for every ground action $\alpha = A(\vec{t})$, 
      \begin{equation}
      \label{eq:gammasimplified}               
      \gamma^{\pm}_F(\vec{x},\alpha,S_0) \Leftrightarrow \exists \vec{z}. \bigl( \vec{t}= \vec{v} \land \phi^\pm_A(\vec{x},\vec{z},S_0) \bigr).
      \end{equation}
      Henceforth, whenever we refer to $\gamma^{\pm}_F(\vec{x},\alpha,S_0)$, we mean the RHS of Eq.~\eqref{eq:gammasimplified} and this applies to $\D_{ss}[\alpha,S_0]$ as well.

\section{Size Complexity of Progression}
\label{sec:SpaceComplexity}

Previous work on the FO progression of acyclic action theories \cite{DBLP:conf/ijcai/0002C24} identified some complexity results in terms of computation time.
However, the resulting theory's size is equally crucial, since it directly affects the computational costs of checking queries against the theory.
Here, we identify upper bounds for the size of FO progressions for BAT-LE, BAT-NR and BAT-AC theories.

Formally, the size $\size{\phi}$ of a formula $\phi$ is 
the number of atoms (including equalities), 
boolean connectives from $\{\land,\lor,\lnot\}$,
and quantifiers from $\{\exists,\forall\}$
it contains.
  As mentioned, a formula with free variables is understood as
  its universal closure. Where possible, leading universal quantifiers are omitted, and do not count towards the size.
  Also, as usual, we identify finite theories with the conjunction of formulas they contain.
  To avoid confusion, when we talk about the size of such a finite theory, we always mean the size of the formula it presents, and {\em not} the number of formulas in the set.

The results in this section will present bounds for the size of the progression, which is understood as the set of sentences $\D_{S_\alpha}$ that replaces $\D_{S_0}$ in the BAT to reflect the state of the world in the new situation after executing ground action $\alpha$.
In particular (cf.~Def.~\ref{def:progression}), the axioms encoding unique names for actions, $\D_{una}$, are not considered part of the progression, but may be used during the construction of a progression for simplification purposes.
We further assume that the vocabulary of action and fluent symbols is fixed and finite.

\subsection{BAT-LE}

The idea behind local-effect action theories \cite{DBLP:conf/ijcai/LiuL09} is to impose constraints on SSAs such that any ground action only affects finitely many objects, namely those explicitly mentioned in the action's arguments. In consequence, only finitely many fluent instances are affected.

Formally, a ground action $\alpha=A(\vec{t})$ is said to have local effects on a fluent $F$ if $\gamma^{\pm}_{F}(\vec{x},\alpha,S_0)$ in Eq.~\eqref{eq:gammasimplified} is of the form  $\exists \vec{z}.[\vec{t}=\vec{v}\land \phi_A^{\pm}(\vec{z},S_0)]$. We call a BAT a BAT-LE (wrt $\alpha$) if $\alpha$ is local-effect to all fluents. For such a BAT, $\D_{una}$ implies 
$\gamma^{\pm}_F(\vec{x},\alpha,S_0) \equiv \bigvee_i (\vec{x} = \vec{t}_i \land \phi_A^{\pm}(\vec{x},S_0))$ (as all possible occurrences of $\vec{x}$ on the RHS of Eq.~\eqref{eq:gammasimplified} are among $\vec{v}$). Hence, there are only finitely many affected fluent instances, i.e., the \emph{characteristic set} $\Omega(s)$ given by
\begin{align*}
\{F(\vec{t},s) \mid (\vec{x} = \vec{t} ~\text{appears in}~ \gamma_F^+(\vec{x},\alpha,s) ~\text{or}~ \gamma_F^-(\vec{x},\alpha,s)\}.
\end{align*}
Let $\D_{ss}[\Omega]$ then be the set of instantiations of $\D_{ss}$ wrt $\Omega$:
\begin{align*}
\D_{ss}[\Omega] = \{F(\vec{t},S_\alpha) \equiv \Phi_F(\vec{t},\alpha,S_0) \mid F(\vec{t},s) \in \Omega(s) \},
\end{align*}
where $\Phi_F(\vec{t},\alpha,S_0)$ is the right-hand side of the SSA for $F$.
Liu and Lakemeyer \shortcite{DBLP:conf/ijcai/LiuL09} observe that
for BAT-LE, forgetting predicates in the progression as in Eq.~\eqref{eq:linreiter} can be done via forgetting the ground atoms in $\Omega$:
If $\theta$ is a consistent, finite set of ground literals, let $\phi[\theta]$ denote the result of replacing every occurrence of an atom $P(\vec{t})$ in $\phi$ by the formula
\begin{align}
\label{eq:LEsubstitute}
\bigvee_{j=1}^k (\vec{t} = \vec{t}_j \land v_j) \lor (\bigwedge_{j=1}^k \vec{t} \neq \vec{t}_j) \land P(\vec{t}),
\end{align}
where for $j=1,\dots,k$, $\theta$ specifies truth value $v_j$ for $P(\vec{t}_j)$.
Theorem 3.6 of \cite{DBLP:conf/ijcai/LiuL09} states that the following is then a progression of $\D_{S_0}$ wrt $\alpha,\D$:
\begin{equation}
\label{eq:forgetlocal}
\bigwedge \D_{una} \land \bigvee_{\theta\in\M(\Omega(S_0))} (\D_{S_0} \cup \D_{ss}[\Omega])[\theta](S_0/S_\alpha)
\end{equation}
Here, $\M(\Omega(S_0))$ denotes the set of all maximal consistent sets of literals over the atoms in $\Omega(S_0)$. The big disjunction in Eq.~\eqref{eq:forgetlocal} essentially amounts to $\forget(\D_{S_0} \cup \D_{ss}[\Omega], \Omega_{S_0})$ (which means iteratively forgetting atoms in $\Omega(S_0)$). The claim in the following theorem was already stated in \cite{DBLP:conf/ijcai/LiuL09}, however, without a proof.
\begin{theorem}
\label{thm:SizeLocalEffect}
If $\D$ is BAT-LE wrt ground action $\alpha$, then its progression can be represented by a FO theory of size
$O(2^c (n+m))$, where
$c$ is the size of the characteristic set,
$n$ is the size of $\D_{S_0}$,
and $m$ is the size of $\D_{ss}$.
\end{theorem}
\begin{proof}
The size of $\D_{S_0}$ is $n$ by assumption, and the size of $\D_{ss}[\Omega]$ is bounded by $cm + c - 1$,
since there are $c$ many formulas in $\D_{ss}[\Omega]$, each of size less or equal to $m$, with $(c-1)$ conjunction symbols between them.
If $\phi$ denotes $\D_{S_0} \cup \D_{ss}[\Omega]$, we hence have $\size{\phi} \leq n + cm + c$.
\par
Observe that for any $\theta\in\M(\Omega(S_0))$, the number of elements in $\theta$ is the same as the number of elements in $\Omega$, i.e., $c$.
In particular, in \eqref{eq:LEsubstitute}, $k \leq c$.
If $a$ denotes the maximal arity of any predicate $P$ in $\phi$, $\vec{t} = \vec{t}_j$ abbreviates a formula of size less than $2a$, and so the size of the formula in \eqref{eq:LEsubstitute} is bounded by $4ca + 4c + 3$.
Therefore, for any $\phi$, $\size{\phi[\theta]} \leq (4ca + 4c + 3)\size{\phi}$.
\par
Finally, observe that the step $(S_0/S_\alpha)$ substitutes one constant by another, and hence does not change the size of the formula.
Also, there are $2^c$ many $\theta$'s in $\M(\Omega(S_0))$.
As argued above, the size of $\D_{una}$ can be ignored.
In summary, we can estimate the size of the progression to be no more than $2^c \cdot (n + cm + c) \cdot (4ca + 4c +3)$.
Since $a$ is constant under the assumption of a fixed vocabulary, we get that the size of the progression is in $O(2^c (n+m))$.
\end{proof}
We remark that $c$ can be bounded by a constant.
Note that the elements of $\Omega[S_0]$ are of the form $F(\vec{t},S_0)$, where $F$ is a local-effect fluent, and the $\vec{t}$ have to be chosen from the arguments of ground action $\alpha$.
Therefore, if $l$ is the number of local-effect fluents, $b$ is the arity of $\alpha$, and $a$ is again the maximal arity among predicates, $c \leq l \cdot b^a$, where all of $l$, $a$, and $b$ are constant under the assumption of a fixed vocabulary.
It can hence also be argued that the size of the progression of a BAT-LE is linear in the size of the original theory.

\subsection{BAT-NR}

To facilitate global effects, BAT-NR relaxes the restriction on SSAs but imposes additional assumptions on $\D_{S_0}$ so that one can use a variant Ackermann's Lemma \cite{ackermann1935untersuchungen} to eliminate SO quantifiers. Formally, a finite theory $T$ is \emph{semi-definitional} wrt a predicate $P$ if
the only occurrence of $P$ in $T$ is of the form $P(\vec{x}) \supset \phi(\vec{x})$ or $\psi(\vec{x}) \supset P(\vec{x})$. $\phi$ (respectively $\psi$) is called a necessary (sufficient) condition of $P$. Let $\mathrm{WSC}_P$ be the set formulas $\psi(\vec{x})$ such that $\psi(\vec{x}) \supset P(\vec{x})$ is in $T$, and $\mathrm{SNC}_P$ be the set of formulas $\phi(\vec{x})$ with $P(\vec{x}) \supset \phi(\vec{x})$ in $T$. Then the following holds:
\begin{theorem}[\thmcite{DBLP:conf/ijcai/LiuL09}]
\label{thm:semidefforget}
Let $T$ be finite and semi-definitional wrt $P$, and $T'$ the set of
sentences in $T$ not mentioning $P$. Then
$\forget(T,P) \Leftrightarrow T' \land \band_{\psi \in \mathrm{WSC}_P, \phi \in \mathrm{SNC}_P} \all
\vec{x}.\psi(\vec{x}) \supset \phi(\vec{x})$.
\end{theorem}

\begin{proposition}
\label{prop:ssatrans}
In any model of $\D$, the SSA for $F$ instantiated by ground action $\alpha$, $\D^F_{ss}[\alpha,S_0]$, is equivalent to:
\begin{subequations}
  \begin{align}
    & \label{eq:ssa1} \neg \gamma^+_{F}(\vec{x},\alpha,S_0) \wedge F(\Vec{x},S_\alpha) \supset F(\Vec{x},S_0) \\
    & \label{eq:ssa2} F(\Vec{x},S_0) \supset \gamma^-_{F}(\vec{x},\alpha,S_0) \vee F(\Vec{x},S_\alpha)  \\
    & \label{eq:ssa3} \gamma^+_F(\vec{x},\alpha,S_0) \supset F(\Vec{x},S_\alpha)  \\
    & \label{eq:ssa4} \gamma^-_F(\vec{x},\alpha,S_0) \supset \neg F(\Vec{x},S_\alpha)   
  \end{align}
\end{subequations}
\end{proposition}
We say a formula $\phi$ is semi-definitional wrt a fluent $F$ if $\phi\uparrow S_0$ is semi-definitional wrt $F$'s lifting predicate. The proposition suggests that $\D^F_{ss}[\alpha,S_0]$ for every fluent $F$ is semi-definitional wrt $F$. Hence, if $\D_{S_0}$ is also semi-definitional wrt fluents, one can use Thm.~\ref{thm:semidefforget} to forget their lifting predicates in Eq.~\eqref{eq:linreiter}, which is the idea behind normal actions and their progression. For an action $\alpha$, we distinguish local-effect fluents $\LEF(\alpha)$ on the one hand (whose SSAs are as in BAT-LE), and non-local effect fluents $\NLEF(\alpha)$ on the other hand.
Normal actions then require that for all $F \in \NLEF(\alpha)$, fluents occurring in $\gamma^{\pm}_F$ can only be in $\LEF(\alpha)$. We call a BAT $\D$ a BAT-NR (wrt action $\alpha$) if $\alpha$ is normal and $\D_{S_0}$ is semi-definitional wrt fluents in $\NLEF(\alpha)$. For such a BAT, one can forget lifting predicates of fluents in $\NLEF(\alpha)$ by Thm.~\ref{thm:semidefforget} and thereafter fluents in $\LEF(\alpha)$ as in Eq.~\eqref{eq:forgetlocal}.

\begin{theorem}
\label{thm:SizeNormal}
If $\D$ is BAT-NR wrt ground action $\alpha$, the result of forgetting the lifting predicates for all fluents in $\NLEF(\alpha)$ in $(\D_{S_0} \cup \D_{ss}[\alpha,S_0])\uparrow S_0$ can be represented by a FO theory whose size is in $O((n+m)^2)$, where
$n$ is the size of $\D_{S_0}$,
and $m$ is the size of $\D_{ss}$.
\end{theorem}
\begin{proof}
First, note that we can ignore the lifting operator $\uparrow S_0$ here, since it only replaces atoms by other atoms, and thus does not affect the size of a formula.
Now, assume that for fluents $F\in\NLEF(\alpha)$, axioms in $\D_{ss}[\alpha,S_0]$ are represented according to Prop.~\ref{prop:ssatrans}.
Note that \eqref{eq:ssa1} and \eqref{eq:ssa2} are semi-definitional wrt $F(\vec{x},S_0)$, and \eqref{eq:ssa3} and \eqref{eq:ssa4} do not mention $F(\vec{x},S_0)$.
Theorem \ref{thm:semidefforget} is hence applicable to $\D_{S_0} \cup \D_{ss}[\alpha,S_0]$.
Intuitively, the set of semi-definitional formulas in $\D_{S_0} \cup \D_{ss}[\alpha,S_0]$ can be represented as a directed graph, where any formula $\psi(\vec{x}) \limp \phi(\vec{x})$ corresponds to an edge from $\psi(\vec{x}) $ to $\phi(\vec{x})$.
Forgetting the lifting predicate for a fluent $F\in\NLEF(\alpha)$ then amounts to eliminating the node for $F(\vec{x},S_0)$, ``bypassing'' all involved edges, i.e., every pair of ingoing edge $\psi(\vec{x}) \limp F(\vec{x},S_0)$ and outgoing edge $F(\vec{x},S_0) \limp \phi(\vec{x})$ is replaced by a new edge $\psi(\vec{x}) \limp \phi(\vec{x})$.
This elimination operation is performed iteratively for all $F\in\NLEF(\alpha)$ (in any order), leaving only nodes that correspond to conditions that do {\em not} mention any of the non-local-effect fluents.
In the worst case, the resulting graph is total, i.e., every condition is connected to every other condition through an edge.
Since the size of all conditions together is bounded by $n+m$, the size of the theory representing all combinations of implications between them is in $O((n+m)^2)$.
Moreover, the size of the original formulas in $\D_{S_0} \cup \D_{ss}[\alpha,S_0]$ that do not mention any of the NLE fluents (including \eqref{eq:ssa3} and \eqref{eq:ssa4}) is bounded by $n+m$ as well.
All in all, the size of the resulting theory is thus in $O((n+m)^2)$.
\end{proof}
\noindent
According to \cite{DBLP:conf/ijcai/LiuL09}, the progression of a BAT-NR through an action $\alpha$ is obtained by first forgetting the (lifting predicates) of the fluents in $\NLEF(\alpha)$, and afterwards forgetting the remaining atoms over the fluents in $\LEF(\alpha)$. We can estimate the size of the resulting theory by combining the results from Theorems \ref{thm:SizeLocalEffect} and \ref{thm:SizeNormal}, giving:
\begin{corollary}
If $\D$ is BAT-NR wrt ground action $\alpha$, then its progression can be represented by a FO theory of size $O(2^c (n+m)^2)$, where
$c$ is the size of the characteristic set,
$n$ is the size of $\D_{S_0}$,
and $m$ is the size of $\D_{ss}$.
\end{corollary}

\subsection{BAT-AC}

While normal action theories allow for non-local effects of a certain form, they include the requirement that for BAT-NR, the effect conditions $\gamma^{\pm}_F$ for a fluent $F$ may only mention local-effect fluents.
This assumption ensures that forgetting the lifting predicate for $F$ will not generate formulas that mention the lifting predicate of another NLE fluent $F'$, which might not have the right (semi-definitional) form to proceed with forgetting $F'$ in the same manner.
However, this excludes domains where actions' effects for a NLE fluent might depend on the state of another NLE fluent. 
BAT-AC, the class of acyclic action theories \cite{DBLP:conf/ijcai/0002C24} overcomes this restriction by allowing NLE fluents to depend on one another, as long the dependencies do not form any cycles.
Here, additional restrictions on the form of $\gamma^\pm_F$ guarantee that the lifting predicates for NLE fluents can be forgotten in a recursive fashion when following the order of dependencies among fluents.
Formally, a formula $\phi(\vec{x})$ is said to be in \emph{good form} wrt a predicate $P$ if it is of the form
  \begin{equation}
      \label{eq:goodform}
      [\psi(\Vec{x}) \vee P(\vec{t}) ] \land [\psi'(\Vec{x}) \vee \neg P(\vec{t})]\land \psi''(\vec{x})
  \end{equation}
  where $\psi,\psi',\psi''$ contain no $P$, and terms in $\vec{t}$
  are either ground terms or free variables among $\Vec{x}$.   
  \begin{proposition}[\thmcite{DBLP:conf/ijcai/0002C24}]
   \label{prop:goodformproperties}   
   If $\phi(\vec{x})$ is in good form wrt $P$, then 
   \begin{enumerate}
       \item $\phi(\vec{x})$ can be rewritten to be semi-definitional wrt $P$;
       \item $\neg \phi(\vec{x})$ can be rewritten to be in good form wrt $P$.
   \end{enumerate}
  \end{proposition}
Now, given SSAs $\D_{ss}$ and a ground action $\alpha$, the \emph{dependency graph} $G$ for $\D_{ss}$ wrt $\alpha$ is a directed graph whose vertices are the set of fluents and there is an edge $F\rightarrow F'$ for two fluents $F,F'$ iff $F \in \mathrm{NLE}(\alpha)$ and  $F'$ appears in $\gamma^{\pm}_F$.
A ground action $\alpha$ is {\em acyclic} (or has acyclic effects) if 1) its dependency graph $G$ is acyclic, i.e., a DAG; and 2) for each fluent $F$, there are \emph{at most two} fluents $F',F''$ in $\mathrm{NLE}(\alpha)$ such that $F \rightarrow F'$ and $F \rightarrow F''$, where $\gamma^+_F$ and $\gamma^-_F$ are in good form wrt $F'$ and $F''$, respectively.

We call a BAT a BAT-AC (wrt $\alpha$), if $\alpha$ is acyclic and $\D_{S_0}$ is \emph{separably semi-definitional} wrt $\NLEF(\alpha)$, that is, each sentence in $\D_{S_0}$ mentions at most one fluent $F$ in $\NLEF(\alpha)$ and is semi-definitional wrt $F$. \cite{DBLP:conf/ijcai/0002C24} showed that 
if a BAT $\D$ is a BAT-AC, the progression of $\D_{S_0}$ wrt $\alpha, \D$ is FO definable.
This is because for any $F,F' \in \NLEF(\alpha)$, even if forgetting the lifting predicate of $F$ will possibly generate new formulas that mention $F'$, the assumption of $\gamma^{\pm}_F$ being in good form wrt $F'$ ensures that the new formulas can be rewritten to be semi-definitional wrt $F'$ again. Hence, one can recursively forget lifting predicates for $F \in \NLEF(\alpha)$ (following the order of the DAG) and later forget lifting predicates of $F \in \LEF(\alpha)$ as in Eq.~\eqref{eq:forgetlocal}.
For analyzing the size of the resulting theory, we need the following:
\begin{definition}
If $\phi(\vec{x})$ is in good form wrt $P$, i.e., 
$\phi(\vec{x}):= (\psi(\vec{x}) \lor P(\vec{t})) \land (\psi'(\vec{x}) \lor \lnot P(\vec{t})) \land \psi''(\vec{x})$,  we define 
$\consize{P}{\phi} = \max \{ \size{\psi}, \size{\psi'}, \size{\psi''} \}$ as  the {\em condition size of $\phi$ wrt $P$}.
\end{definition}
As we will see, the size of conditions within good-form formulas will be the determining factor for the size complexity of a progression for a BAT-AC.
Furthermore, we assume that for each $F \in \NLEF(\alpha)$, $\D_{S_0}$ contains
{\em exactly one} axiom 
\begin{align*}
\Delta_F = \left( \NWSC_F(\vec{x}) \lor F(\vec{x},S_0) \right) \land
\left( \SNC_F(\vec{x}) \lor \lnot F(\vec{x},S_0) \right)
\end{align*}
This is without loss of generality: If $\D_{S_0}$ contains multiple semi-definitional formulas for $F$, let $\NWSC_F$ be the negation of the weakest sufficient condition, and $\SNC_F$ the strongest necessary condition.
Observe that with the implicit negation in $\NWSC$, $\Delta_F$ is indeed in semi-definitional form, and such formulas can be assumed to be of good form \eqref{eq:goodform}.
Also, this assumption has no impact on size or condition size:
\begin{proposition}
\label{prop:CondSizeAnd}
  If $\phi_1(\vec{x}),\phi_2(\vec{x}),\dots,\phi_n(\vec{x})$ are all
  semi-definitional wrt predicate $P$, then
  $\phi_1(\vec{x}) \land \phi_2(\vec{x}) \land \dots \land
  \phi_n(\vec{x})$ can be rewritten to a single equivalent formula
  $\varphi(\vec{x})$ such that
  $\size{\varphi(\vec{x})} \leq \size{\bigwedge_i\phi_i(\vec{x})}$ and
  $\consize{P}{\varphi(\vec{x})} \leq \consize{P}{\bigwedge_i\phi_i(\vec{x})}$.
\end{proposition}
Next, we consider formulas in good form and how rewriting them as in Prop.~\ref{prop:goodformproperties} affects the size of the result:
\begin{proposition}
\label{prop:SizeNotGF}
  If $\phi(\vec{x})$ is in good form wrt $P$, then
\begin{enumerate}
  \item $\phi(\vec{x})$ can be
  rewritten to $\varphi(\vec{y})$ that is semi-definitional
  wrt $P(\vec{y})$, moreover, if $a=\max\{\size{\vec{x}},\size{\vec{y}}\}$, then  $\size{\varphi(\vec{y})} \leq \size{\phi(\vec{x})} + 7a + 2$ and
  $\consize{P}{\varphi(\vec{y})} \leq \consize{P}{\phi(\vec{x})} + 3a + 1$;
    
    \item $\lnot\phi(\vec{x})$ can be rewritten to 
  $\varphi(\vec{x})$ that is in good form wrt $P$ with
  $\size{\varphi(\vec{x})} \leq 3 \size{\phi(\vec{x})}$ and
  $\consize{P}{\varphi(\vec{x})} \leq 3 \consize{P}{\phi(\vec{x})} + 2$;

  \item if $\psi(\vec{x})$ mentions no $P$, then $\phi(\vec{x}) \lor \psi(\vec{x})$
  can be rewritten to a formula $\varphi(\vec{x})$ in good form wrt
  $P$, where
  $\size{\varphi(\vec{x})} = \size{\phi(\vec{x})} + 3
  \size{\psi(\vec{x})} + 3$ and
  $\consize{P}{\varphi(\vec{x})} = \consize{P}{\phi(\vec{x})} +
    \size{\psi(\vec{x})} + 1$.
\end{enumerate}
\end{proposition}

Then we have:

\begin{theorem}
\label{thm:acyclicSizeComplexity}
If $\D$ is BAT-AC wrt ground action $\alpha$, the result of forgetting the lifting predicates for all fluents in $\NLEF(\alpha)$ in $(\D_{S_0} \cup \D_{ss}[\alpha,S_0])\uparrow S_0$ can be represented by a FO theory whose size is in $O(2^d(n+m))$, where
$d$ is the depth of the dependency graph,
$n$ the size of $\D_{S_0}$,
and $m$ the size of $\D_{ss}$.
\end{theorem}
\begin{proof}
The proof for Theorem 3 in \cite{DBLP:conf/ijcai/0002C24} makes an inductive argument that one can forget fluents $F\in\NLEF(\alpha)$ iteratively while maintaining the required form, as long as one forgets all predecessors of $F$ in the dependency graph before forgetting $F$ itself.
By Prop.~\ref{prop:ssatrans} and Thm.~\ref{thm:semidefforget}, forgetting any such $F$ amounts to replacing the corresponding $\Delta_F$ by the following (parameters of $\gamma^{\pm}_{F}$ are omitted):
\begin{subequations}
  \begin{align}
    &\label{eq:ProgAxC} \NWSC_F(\vec{x}) \lor \SNC_F(\vec{x}) && \\
    &\label{eq:ProgAxA} \lnot \gamma_F^+ \lor F(\vec{x},S_\alpha) && (3k + 3a + 5)\\
    &\label{eq:ProgAxB} \lnot \gamma_F^- \lor \lnot F(\vec{x},S_\alpha) && (3k + 3a + 6)\\
    &\label{eq:ProgAxD} \gamma_F^+ \lor \lnot F(\vec{x},S_\alpha) \lor \SNC_F(\vec{x}) && (w_F + k + 3a + 5)\\
    &\label{eq:ProgAxE} \NWSC_F(\vec{x}) \lor \gamma_F^- \lor F(\vec{x},S_\alpha) && (w_F + k + 3a + 4)
  \end{align}
\end{subequations}
By definition, there are at most two other NLE fluents $F_1,F_2$ with $F \rightarrow F_1$ and $F \rightarrow F_2$, where $F_1$ only occurs in $\gamma_F^+$, and $F_2$ only in $\gamma_F^-$.
\eqref{eq:ProgAxA} and \eqref{eq:ProgAxD} will subsequently be incorporated into the new axiom $\Delta_{F_1}$, and \eqref{eq:ProgAxB} and \eqref{eq:ProgAxE} into the new axiom $\Delta_{F_2}$.
If $w_F$ is the condition size of the updated $\Delta_F$ and
$k$ the maximal size of an SSA for any NLE fluent, then by Prop.~\ref{prop:SizeNotGF}, upper bounds for the condition sizes of formulas \eqref{eq:ProgAxA}--\eqref{eq:ProgAxE} after rewriting them to be in good form wrt.~$F_1,F_2$ are as stated above in parentheses.
If $w$ is an upper bound for the condition size of any {\em original} $\Delta_F$ axiom, the condition size for the {\em updated} $\Delta_{F'}$  is bounded according to 
\begin{align}
\label{eq:SizeRecurrence}
w_{F'} &\leq w + \sum_{F \rightarrow F'} (w_F + 4k + 6a + 10).
\end{align}
The sum of the sizes of \eqref{eq:ProgAxC}--\eqref{eq:ProgAxE} over all NLE fluents $F$ gives an upper bound for the size of the progression.
Observe that the total size of $\gamma_F^\pm$, $F(\vec{x},S_\alpha)$ atoms, and boolean connectives is less than $\size{\D_{ss}^F[\alpha,S_0]}$ as in Prop.~\ref{prop:ssatrans},
i.e., less than $m$ summed over all NLE fluents.
This leaves two occurrences of $\NWSC_F(\vec{x})$ and two of $\SNC_F(\vec{x})$, each of size $\leq w_F$, summing up to $S = 4 \sum_{F\in\NLEF(\alpha)} w_F$.
To determine the sum of all $w_F$, we use the intuition that the linear 
recurrence in \eqref{eq:SizeRecurrence} can be unrolled along all source-to-node
paths.
Intuitively, each node $F$ starts with a $w_F$ ``weight'' of $w$ and adds $w_F + 4k + 6a + 10$ to each of its children.
This contribution is added multiple times, namely for each path starting in $F$.
As the out-degree of nodes is $\leq 2$, the number of distinct paths starting in $F$ is $\leq 2^{d+1}$.
If $l$ is the number of NLE fluents, we get $S \leq 4 \cdot 2^{d+1} \cdot l \cdot (w + 4k + 6a + 10)$.
Since $a$ is a constant, $l \cdot w \in O(n)$, and $l \cdot k \in O(m)$, the claim follows.
\end{proof}
\noindent
Combining the above theorem with Theorem \ref{thm:SizeLocalEffect} again yields:
\begin{corollary}
If $\D$ is BAT-AC wrt ground action $\alpha$, then its progression can be represented by a FO theory of size $O(2^{c+d} (n+m))$, where
$c$ is the size of the characteristic set,
$d$ is the depth of the dependency graph,
$n$ is the size of $\D_{S_0}$,
and $m$ is the size of $\D_{ss}$.
\end{corollary}
Observe that the progression for BAT-LE and BAT-AC is linear in the size of the theory, while that for BAT-NR is quadratic.
As mentioned previously, the three classes roughly form a hierarchy of increasing expressiveness, but not exactly.
This is because BAT-NR allow ``mixed'' semi-definitional axioms of the form $F(\vec{x},S_0) \limp G(\vec{x},S_0)$, where both $F$ and $G$ are NLE fluents, whereas BAT-AC requires a strict separation so that a semi-definitional axiom for $F$ does not mention any NLE fluent $G$, and vice versa.
If we impose this additional restriction on BAT-NR, the resulting class will indeed be a proper subset of BAT-AC, consisting of action theories whose dependency graph has a depth of at most one, and where every edge $F \rightarrow G$ is between an NLE fluent $F$ and a LE fluent $G$.
The size of the progression then is in $O(2^c (n+m))$.

\section{The Two-Variable Fragment}
\label{sec:FOTwo}
For practical applications like FO planning, it is not enough to have an FO progression, but the progressed KB has to admit effective query evaluation. This necessitates a decidable progressed theory. A class of theory $\mathcal{T}$ is called \emph{decidable} if for any $T, \phi \in \mathcal{T}$, it is decidable to check if $T\models \phi$. One popular decidable class of FO theories is the {\em two-variable fragment} with equality, called $\FOtwo$ \cite{DBLP:journals/bsl/GradelKV97}. It consists of all formulas with at most 2 variable symbols, which we denote as $x$ and $y$ by convention. By $\vec{x}$ we then mean a vector of variables that either consists only of $x$ or of $x$ and $y$.

\begin{theorem}[\thmcite{DBLP:journals/bsl/GradelKV97}]
  $\FOtwo$ has the exponential model property: there is a constant $c$
 s.t. every satisfiable $\FOtwo$-sentence $\phi$ has a model of
  cardinality at most $2^{c |\phi|}$, where $|\phi|$ is the size of $\phi$.
\end{theorem}
\noindent Consequently, satisfiability of $\FOtwo$ is NEXPTIME-complete and entailment checking is CoNEXPTIME-complete.

Typically, $\FOtwo$ does not include general function symbols. Yet, in some literature, constant symbols are allowed as they can be simulated through unary predicates whose extensions consist of exactly one element. That is to say, for any constant $A$, one introduces a predicate $P_A(x)$ and includes an axiom
\[ \exists x P_A(x) \land \forall x \forall y.\, P_A(x) \land P_A(y)
  \limp (x = y) \] 
saying that there is exactly one domain element in the extension of $P_A$.
It is easy to see that this idea can be extended to also include ground situation terms $\sdo(\alpha,S_0)$.

Moreover, in the following, we assume that every rigid relation symbol has arity at most two. Again, relations of higher arity can be simulated by introducing new predicate symbols, with formulas growing at most linearly in size. For details, we refer the reader to \cite{DBLP:journals/bsl/GradelKV97}. For fluent predicates, we allow for up to two {\em object} arguments in addition to the {\em situation} argument. The idea here is that the situation argument is irrelevant, as it will be instantiated by some ground situation term when evaluating queries.

\begin{definition}
Given a ground action $\alpha$, a BAT $\D$ is an $\FOtwo$-BAT wrt $\alpha$ if
$\D_{S_0} \cup \D_{ss}[\alpha,S_0]$ is a finite $\FOtwo$
theory.
\end{definition}
In what follows, we show that the progressions of $\FOtwo$-BATs wrt local-effect, normal, and acyclic actions are all within the $\FOtwo$ fragment, hence decidable. 

\paragraph{$\FOtwo$-BAT-LE} To begin with, we consider $\FOtwo$-BAT where the action $\alpha$ has local effects, i.e., $\FOtwo$-BAT-LE.

\begin{theorem}
\label{thm:fo2leclosed}
Given a BAT $\D$ that is BAT-LE wrt ground action $\alpha$, if $\D$ is also an $\FOtwo$-BAT wrt $\alpha$, then the progression of $\D_0$ wrt $\D,\alpha$ can be expressed in $\FOtwo$.
\end{theorem}

It is not hard to check that $(\D_{S_0} \cup \D_{ss}[\Omega])[\theta]$ in Eq.~\eqref{eq:forgetlocal} is in $\FOtwo$ by assumption. Besides, a disjunction of $\FOtwo$ theories is still in $\FOtwo$, so the progression as in Eq.~\eqref{eq:forgetlocal} is in $\FOtwo$.

\paragraph{$\FOtwo$-BAT-NR} For $\FOtwo$-BAT with normal actions, we have   

\begin{theorem}
\label{thm:fo2nrclosed}
Given a BAT $\D$ that is BAT-NR wrt ground action $\alpha$, if $\D$ is an \FOtwo-BAT wrt $\alpha$, then the progression of $\D_0$ wrt $\D,\alpha$ can be expressed in $\FOtwo$.
\end{theorem}

For BAT-NR, forgetting NLE fluents as in Thm.~\ref{thm:semidefforget}
involves combining sufficient conditions and necessary conditions of semi-definitional formulas, which clearly does not introduce new variables. Hence, the result is still in $\FOtwo$. Besides, for LE fluents, as mentioned before, forgetting them only involves disjoining two $\FOtwo$ theories, which is still in $\FOtwo$. Hence, the resulting progression is in $\FOtwo$.

\paragraph{$\FOtwo$-BAT-AC} The situation for $\FOtwo$-BAT with acyclic actions is 
slightly more complicated, as we need to ensure that the generated formulas when forgetting an NLE fluent are semi-definitional wrt another NLE fluent and the rewrite has to be in $\FOtwo$. The following lemma justifies this.  

\begin{lemma}
\label{prop:fo2goodformproperties}
  If $\phi(\vec{x}) \in \FOtwo$ is in good form wrt $P$, then
  \begin{enumerate}
      \item $\phi(\vec{x})$ can be rewritten to a formula in $\FOtwo$ that is semi-definitional wrt $P$; 
      
      \item $\neg \phi(\vec{x})$ can be rewritten to a formula in $\FOtwo$ that is in good form wrt $P$. 
  \end{enumerate}
\end{lemma}
The proof of Item 2 is simple, as the rewrite only involves pushing negation inward and distributing conjunctions over disjunctions, operations that preserve membership in $\FOtwo$. For Item 1, let $\phi(\vec{x}) =
 \{ \all \Vec{x} (\psi(\Vec{x}) \vee P(\vec{t})),
\all \Vec{x}(\psi'(\Vec{x}) \vee \neg P(\vec{t})),
\all \Vec{x} \psi''(\Vec{x}) \}$.  By assumption, $\vec{x}$ denotes either only $x$ or $x$ and $y$. Since $P$ is unary or binary, $\vec{t}$ has length one or two, and
since functions are excluded, $\vec{t}$ can only consist of a
combination of $x$, $y$, and some constants $A$ and $B$. Hence, for
$\forall \vec{x} (\psi(\vec{x}) \vee P(\vec{t}))$ (likewise for the $\neg P$ case), we have cases:
\begin{itemize}
\item $\forall x (\psi(x) \lor P(A)) \Leftrightarrow \forall x ((x=A) \land \exists x \lnot\psi(x)) \limp P(x)$.
\item $\forall x (\psi(x) \lor P(x))$ has the right form already.
\item $\forall x \forall y (\psi(x,y) \lor P(A,B))$ can be rewritten
  as\\
  $\forall x \forall y ((x=A \land y=B) \land \exists x \exists y
  \lnot \psi(x,y)) \limp P(x,y)$.
\item $\forall x \forall y (\psi(x,y) \lor P(x,B))$ can be rewritten
  as\\
  $\forall x \forall y ((y=B) \land \exists y \lnot\psi(x,y)) \limp
  P(x,y)$.
\item $\forall x \forall y (\psi(x,y) \lor P(A,y))$ can be rewritten
  as\\
  $\forall x \forall y ((x=A) \land \exists x \lnot \psi(x,y)) \limp
  P(x,y)$.
\item $\forall x \forall y (\psi(x,y) \lor P(x,y))$ has the right form
  already.
\end{itemize}
Hence $\forall \vec{x} (\psi(\vec{x}) \vee P(\vec{t}))$, thus $\phi(\vec{x})$, can be expressed in $\FOtwo$.

\begin{theorem}
\label{thm:fo2acclosed}
Given a BAT $\D$ that is BAT-AC wrt ground action $\alpha$, if $\D$ is a $\FOtwo$-BAT wrt $\alpha$, then the progression of $\D_0$ wrt $\D,\alpha$ can be expressed in $\FOtwo$.
\end{theorem}

We remark that the results in this section are not obvious in the following sense.
$\FOtwo$ is a fragment that fails to have the Craig interpolation and Beth definability properties \cite{DBLP:conf/lics/JungW21}.
Among other things, this means that while forgetting a predicate $P$ from a $\FOtwo$ theory might be possible (i.e., yields an equivalent FO theory), it can happen that the resulting theory is not expressible in $\FOtwo$.
Hence, not every $\FOtwo$-BAT has a progression that is in $\FOtwo$, but for the fragments we consider here, this can be guaranteed.

\section{Universal Theories with Constants}
\label{sec:UTC}

Another decidable fragment of FOL is given by universal theories with constants (UTC). In fact, it is shown that Bernays-Sch\"{o}nfinkel \shortcite{Bernays1928ZumED}  (BS) theories are decidable, i.e., theories with sentences of the form $
\exists x_1\ldots\exists x_k \forall y_1 \cdots \forall y_m \phi$,
where $\phi$ is a quantifier-free formula (potentially with equalities) mentioning no functions except constants. Satisfiability checking of BS theories is NEXPTIME-complete. UTC theories are a special case of BS theories where all sentences are universally quantified, i.e., of the form $\forall y_1 \cdots \forall y_m \phi$. For $\UTC$ theory $T$ and query $\phi$, deciding if $T \models \phi$ can be converted to check if the BS theory $T \union \{\neg \phi\}$ is unsatisfiable, hence the complexity is CoNEXPTIME-complete.

\begin{definition}
Given a ground action $\alpha$, a BAT $\D$ is a $\UTC$-BAT wrt $\alpha$ if $\D_{S_0} \cup \D_{ss}[\alpha,S_0]$ is a finite $\UTC$
theory.
\end{definition}
In what follows, we show that the progressions of $\UTC$-BATs wrt local-effect, normal, and acyclic actions are all within the $\UTC$ fragment, hence decidable. 

\paragraph{$\UTC$-BAT-LE} First, for local effect actions, we have that

\begin{theorem}
\label{thm:utcleclosed}
Given a BAT $\D$ that is BAT-LE wrt ground action $\alpha$, if $\D$ is an $\UTC$-BAT wrt $\alpha$, then the progression of $\D_0$ wrt $\D,\alpha$ can be expressed in $\UTC$.
\end{theorem}

Again, the key here is to show that the disjunction of UTC theories can be expressed in $\UTC$. This is indeed the case by applying the following rule to properly rename variables $
  \forall \vec{x} \phi(\vec{x}) \vee \forall \vec{x} \psi(\vec{x}) \Leftrightarrow  
  \forall \vec{x} \forall \vec{y} (\phi(\vec{x}) \vee \psi(\vec{y}))
  $.

\paragraph{$\UTC$-BAT-NR} Likewise, for normal actions, we have

\begin{theorem}
\label{thm:utcnrclosed}
Given a BAT $\D$ that is BAT-NR wrt ground action $\alpha$, if $\D$ is an $\UTC$-BAT wrt $\alpha$, then the progression of $\D_0$ wrt $\D,\alpha$ can be expressed in $\UTC$.
\end{theorem}

The proof is like its counterpart in $\FOtwo$, it suffices to show that forgetting NLE fluents preserves membership in $\UTC$. This is indeed the case as the sufficient and necessary conditions to be combined when forgetting NLE fluents, as in Thm.~\ref{thm:semidefforget}, are all quantifier-free by assumption; hence, the result is still quantifier-free, leading to a $\UTC$ FO progression.

\paragraph{$\UTC$-BAT-AC} Lastly, for acyclic actions, we have that
\begin{lemma}
\label{prop:utcgoodformproperties}
  If $\phi(\vec{x}) \in \UTC$ is in good form wrt $P$, then
  \begin{enumerate}
      \item $\phi(\vec{x})$ can be rewritten to a formula in $\UTC$ that is semi-definitional wrt $P$; 
      
      \item $\neg \phi(\vec{x})$ can be rewritten to a formula in $\UTC$ that is in good form wrt $P$. 
  \end{enumerate}
\end{lemma}
Again, Item 2 is trivial as the rewrite only involves pushing negation inward and distributing conjunctions over disjunctions, operations that preserve membership in $\UTC$. We prove the first item. For Item 1, let
$\phi(\vec{x}) = \{ \all \Vec{x} (\psi(\Vec{x}) \vee P(\vec{t})),
\all \Vec{x}(\psi'(\Vec{x}) \vee \neg P(\vec{t})),
\all \Vec{x} \psi''(\Vec{x}) \}$. $\psi(\Vec{x}),\psi'(\Vec{x}),\psi''(\Vec{x})$ are all quantifier-free and $\vec{t}$ can only contain variables in $\vec{x}$ by assumption. $\all \Vec{x} (\psi(\Vec{x}) \vee P(\vec{t}))$ (likewise for the $\neg P$ case) can be rewritten as $\all \Vec{x} \all \Vec{y}.\vec{y} = \Vec{t} \supset (\psi(\Vec{x}) \vee P(\vec{y}))$, which is in UTC and semi-definitional wrt $P$, and hence so can $\phi(x)$.

\begin{theorem}
\label{thm:utcacclosed}
Given a BAT $\D$ that is BAT-AC wrt ground action $\alpha$, if $\D$ is an $\UTC$-BAT wrt $\alpha$, then the progression of $\D_0$ wrt $\D,\alpha$ can be expressed in $\UTC$.
\end{theorem}

We remark that the fact that the progression of a UTC-BAT can be expressed by another UTC theory here is not so surprising, as \cite{DBLP:conf/kr/Arenas18} showed that the progression of a UTC-BAT wrt \emph{any actions} can be expressed by another UTC theory. In fact, a resolution-like algorithm is proposed there to compute such a progression. Nevertheless, the progression might not be finite, and deciding if a finite progression exists is undecidable; namely, the proposed algorithm might not terminate. In this regard, our contribution here is that we show that if the actions are furthermore local-effect, normal, or acyclic, finite progressions of the respective UTC-BATs actually exist (and can be computed accordingly). In addition, due to our results in Section \ref{sec:SpaceComplexity}, the size of such finite progressions grows only polynomially, even linearly.

\section{Related Work and Discussion}
\label{sec:RelatedWork}

In this paper, we studied the size complexity and decidability of FO progressions in terms of local-effect, normal, and acyclic actions.
Previous contributions along similar lines include those due to Liu and Lakemeyer \shortcite{DBLP:conf/ijcai/LiuL09} and due to Liu and Cla{\ss}en \shortcite{DBLP:conf/ijcai/0002C24}, who showed that a progression can be efficiently computed for normal and acyclic theories, respectively, under the restrictions that the instantiated SSAs are quantifier-free and that the initial theory is in so-called {\em $\properplus$} form.
Intuitively, $\properplus$ KBs \cite{DBLP:conf/kr/Lakemeyer02} correspond to consistent, potentially infinite sets of ground clauses.
What makes them particularly interesting is that certain queries can be computed in polynomial time \cite{DBLP:conf/aaai/LiuL05}.
De Giacomo et al.~\shortcite{de2016progression} showed that progression is FO definable for Situation Calculus theories where the number of fluent instances affected by an action is {\em bounded}, in which case verification (and hence projection) is decidable.
Arenas et al.~\shortcite{DBLP:conf/kr/Arenas18} investigate the feasibility and complexity of progression for UTC-based action theories (without the additional restriction of SSAs to local-effect, normal or acyclic form as done here).
Eiter and Sold{\`a} \shortcite{DBLP:conf/ijcai/EiterS24} studied computational aspects of progression in the context of temporal equilibrium logic.

More broadly, related work includes that by Gu and Soutchanski \shortcite{DBLP:journals/amai/GuS10} who studied the complexity of {\em regression} for action theories over the two-variable fragment of FOL with counting quantifiers as well as description logics $\ALCOU$ and $\ALCQOU$.
Lakemeyer and Levesque \shortcite{DBLP:conf/kr/LakemeyerL14} presented a variant of the modal epistemic Situation Calculus based on a model of limited belief that admits a decidable, yet incomplete form of reasoning.
Calvanese et al.~\shortcite{DBLP:conf/ijcai/CalvaneseGS15} proved that Situation Calculus theories extended with description logic TBoxes acting as state constraints lead to undecidability of the satisfiability problem, even for the simplest kinds of description logics, and the simplest kind of action theories.
When TBoxes are merely viewed as part of the initial theory rather than state constraints, Zarrie{\ss} \shortcite{benjamin-thesis} showed that the verification problem (which subsumes the projection problem) for action formalisms based on various description logics becomes decidable, providing exact complexities.
Note that while many description logics can be viewed as fragments of \FOtwo, related results on projection are not directly applicable to a \FOtwo-based Situation Calculus since translating \FOtwo into description logic leads to a (provably unavoidable) exponential blow-up \cite{DBLP:conf/dlog/LutzSW01}.

Possible directions for future work include exploring whether more expressive classes, or less expressive but more tractable classes, admit similar closure properties as presented here.
Our results also suggest the possibility to use progression instead of regression for the purpose of decidable verification of Golog as in \cite{DBLP:conf/aaai/ZarriessC16}.
Our main motivation however is in practical applications, i.e., equipping real-world agents with the capacity to plan and act in the presence of incomplete information and unbounded domains.
There is a recent trend in AI planning to work directly on FO representations \cite{DBLP:conf/ijcai/CorreaG24}, but merely as a way to avoid the bottleneck of grounding over the given set of objects, which is still assumed to be finite.
When it comes to the full expressiveness of Situation Calculus (within decidability bounds), there is currently a lack of concrete implementations and corresponding benchmarks.
We believe that our work here lays the foundations towards such planning systems that work in truly open worlds.

\ifthenelse{\boolean{shortver}}{}{%
\section*{Appendix}
\subsection*{Proof of Theorem~\ref{thm:acyclicSizeComplexity}}
First, we provide a detailed proof for Proposition~\ref{prop:SizeNotGF}.

\begin{propositionrestated}[Restatement of Proposition~\ref{prop:SizeNotGF}]

  If $\phi(\vec{x})$ is in good form wrt $P$, then
\begin{enumerate}
  \item $\phi(\vec{x})$ can be
  rewritten to $\varphi(\vec{y})$ that is semi-definitional
  wrt $P(\vec{y})$, moreover, if $a=\max\{\size{\vec{x}},\size{\vec{y}}\}$, then  $\size{\varphi(\vec{y})} \leq \size{\phi(\vec{x})} + 7a + 2$ and
  $\consize{P}{\varphi(\vec{y})} \leq \consize{P}{\phi(\vec{x})} + 3a + 1$;
    
    \item $\lnot\phi(\vec{x})$ can be rewritten to 
  $\varphi(\vec{x})$ that is in good form wrt $P$ with
  $\size{\varphi(\vec{x})} \leq 3 \size{\phi(\vec{x})}$ and
  $\consize{P}{\varphi(\vec{x})} \leq 3 \consize{P}{\phi(\vec{x})} + 2$;

  \item if $\psi(\vec{x})$ mentions no $P$, then $\phi(\vec{x}) \lor \psi(\vec{x})$
  can be rewritten to a formula $\varphi(\vec{x})$ in good form wrt
  $P$, where
  $\size{\varphi(\vec{x})} = \size{\phi(\vec{x})} + 3
  \size{\psi(\vec{x})} + 3$ and
  $\consize{P}{\varphi(\vec{x})} = \consize{P}{\phi(\vec{x})} +
    \size{\psi(\vec{x})} + 1$.
\end{enumerate}
\end{propositionrestated}

\begin{proof}
Let $\phi(\vec{x}) =
                    (\psi(\vec{x}) \lor P(\vec{t})) \land (\psi'(\vec{x}) \lor \lnot P(\vec{t})) \land \psi''(\vec{x})$,
    \begin{enumerate}
        \item 
$\phi(\vec{x})$ is equivalent to $\varphi(\vec{y})$ given by 
\begin{equation}
    \label{eq:goodformtranssemiuniversally}
\begin{aligned}
     &((\forall \vec{x}.\, \vec{y} \neq \vec{t} \lor \psi(\vec{x})) \lor P(\vec{y})) ~\land\\
     &((\forall \vec{x}.\, \vec{y} \neq \vec{t} \lor \psi'(\vec{x})) \lor \lnot P(\vec{y}))
     ~\land~ (\forall \vec{x}.\, \psi''(\vec{x}))
\end{aligned}
\end{equation}
  Hence, the size increases by 3 occurrences of $\forall \vec{x}$
  (each size $l$), 2 occurrences of $(\vec{y} \neq \vec{t})$ (each
  size 2$l$), and two additional disjunction symbols;    
        \item $\lnot\phi(\vec{x})$ is equivalent to a formula $\varphi(\vec{x})$ given by

     $(\lnot\psi(\vec{x}) \lor \lnot\psi''(\vec{x}) \lor P(\vec{t})) ~\land
     (\lnot\psi'(\vec{x}) \lor \lnot\psi''(\vec{x}) \lor \lnot P(\vec{t}))~\land
     (\lnot\psi(\vec{x}) \lor \lnot\psi'(\vec{x}) \lor \lnot\psi''(\vec{x}))$
 
  Let $l_1 = \size{\psi(\vec{x})}$, $l_2 = \size{\psi'(\vec{x})}$, and
  $l_3 = \size{\psi''(\vec{x})}$. Then
  $\size{\phi(\vec{x})} = l_1 + l_2 + l_3 + 7$, and
  $\size{\varphi(\vec{x})} = 2 l_1 + 2 l_2 + 3 l_3 + 18$.  Also,
  $\consize{P}{\phi(\vec{x})} = \max \{l_1,l_2,l_3\}$, and every 
  condition in $\varphi(\vec{x})$ has size
  $\leq l_1+l_2+l_3+2 \leq \consize{P}{\varphi(\vec{x})}+2$;
  \item $\phi(\vec{x}) \lor \psi(\vec{x})$ is equivalent to  
  $\varphi(\vec{x})$ given by

     $(\psi(\vec{x}) \lor \psi'(\vec{x}) \lor P(\vec{t})) ~\land
     (\psi(\vec{x}) \lor \psi''(\vec{x}) \lor \lnot P(\vec{t}))~\land
    (\psi(\vec{x}) \lor \psi'''(\vec{x}))$.  
  Observe that the new formula adds 3 occurrences of $\psi(\vec{x})$
  and 3 disjunction symbols to $\phi(\vec{x})$.
    \end{enumerate}
\end{proof}

\begin{theoremrestated}[Restatement of Theorem~\ref{thm:acyclicSizeComplexity}]
If $\D$ is BAT-AC wrt ground action $\alpha$, the result of forgetting the lifting predicates for all fluents in $\NLEF(\alpha)$ in $(\D_{S_0} \cup \D_{ss}[\alpha,S_0])\uparrow S_0$ can be represented by a FO theory whose size is in $O(2^d(n+m))$, where
$d$ is the depth of the dependency graph,
$n$ the size of $\D_{S_0}$,
and $m$ the size of $\D_{ss}$.
\end{theoremrestated}
\begin{proof}
The proof for Theorem 3 in \cite{DBLP:conf/ijcai/0002C24} makes an inductive argument that it is possible to iteratively forget fluents in $\NLEF(\alpha)$ of decreasing fluent depth while maintaining the required form. Looking at an arbitrary $F$, this means that we forget all its predecessors $F^*$ before forgetting $F$ itself. By induction assumption, none of the $F^*$ have any predecessors to consider, and there is a single semi-definitional axiom of the form
\begin{align*}
\Delta_F = \left( \NWSC_F(\vec{x}) \lor F(\vec{x},S_0) \right) \land
\left( \SNC_F(\vec{x}) \lor \lnot F(\vec{x},S_0) \right)
\end{align*}
in the initial KB for each of them (let $\Delta_F$ denote the axiom for $F$ henceforth). Hence, to forget $F$, one only needs to consider $\Delta_F \land \D^F_{ss}[\alpha,S_0]$, which by Prop. 6 amounts to replacing the original $\Delta_F$ by the following (parameters of $\gamma^{\pm}_{F}$ are omitted, expressions in parentheses are explained further below):
\begin{subequations}
  \begin{align}
    &\tag{\ref{eq:ProgAxC}} \NWSC_F(\vec{x}) \lor \SNC_F(\vec{x}) && \\
    &\tag{\ref{eq:ProgAxA}}   \lnot \gamma_F^+ \lor F(\vec{x},S_\alpha) && (3k + 4)\\
    &\tag{\ref{eq:ProgAxB}} \lnot \gamma_F^- \lor \lnot F(\vec{x},S_\alpha) && (3k + 5)\\
    &\tag{\ref{eq:ProgAxD}} \gamma_F^+ \lor \lnot F(\vec{x},S_\alpha) \lor \SNC_F(\vec{x}) && (w_F + k + 4)\\
    &\tag{\ref{eq:ProgAxE}} \NWSC_F(\vec{x}) \lor \gamma_F^- \lor F(\vec{x},S_\alpha) && (w_F + k + 3)
  \end{align}
\end{subequations}
Formula \eqref{eq:ProgAxC} will remain unchanged when forgetting subsequent fluents (recall that $\NWSC_F$ and $\SNC_F$ do not mention any NLE fluents). By definition, there are at most two other NLE fluents $F'_+,F'_-$  in the dependency graph such that $F \rightarrow F'_+$ and $F \rightarrow F'_-$, where $F'_+$ only occurs in $\gamma_F^+$, and $F'_-$ only in $\gamma_F^-$.
Formulas \eqref{eq:ProgAxA} and \eqref{eq:ProgAxD} will subsequently be incorporated into the new KB axiom $\Delta_{F'_+}$, and formulas \eqref{eq:ProgAxC} and \eqref{eq:ProgAxE} into the new axiom $\Delta_{F'_-}$.

We now estimate the condition size of $\Delta_{F'}$ in terms of the conditions size of $\Delta_F$ for $F \rightarrow F'$.
For this purpose, let $w_F$ denote the condition size of the {\em updated} $\Delta_F$ for any $F$, after its predecessors have been forgotten, but before forgetting $F$ itself.
Furthermore, let $k$ be the maximal size of an SSA for any NLE fluent, and $w$ be the maximal condition size of any $\Delta_F$ in the original KB.
Using Prop.~\ref{prop:SizeNotGF}, this gives us upper bounds for the condition sizes of formulas \eqref{eq:ProgAxA}--\eqref{eq:ProgAxE} as stated above in parentheses.
Observe that the sum of condition sizes of \eqref{eq:ProgAxA} and \eqref{eq:ProgAxD} for $F'_+$ is the same as the sum of condition sizes for \eqref{eq:ProgAxB} and \eqref{eq:ProgAxE} for $F'_-$, namely $w_F + 4k + 8$.
Also by Prop.~\ref{prop:SizeNotGF}, converting each of them to semi-definitional form adds at most $(3a+1)$ to their individual condition size, where $a$ is the maximal arity among all fluents.
The original $\Delta_{F'}$ already have the right form, and their condition size is bounded by $w$.
Putting these together, by Prop.~\ref{prop:CondSizeAnd} we have that the condition size for the updated $\Delta_{F'}$ is bounded according to
\begin{align}
w_{F'} &\leq w + \sum_{F \rightarrow F'} (w_F + 4k + 6a + 10). \tag{\ref{eq:SizeRecurrence}}
\end{align}
Note that this bound also holds for fluents without predecessors in the dependency graph; in this case, the sum is empty and only $w$ remains.

For determining the size of the resulting KB, note that formulas of the form \eqref{eq:ProgAxC} are added for all NLE fluents.
Formulas of the form \eqref{eq:ProgAxB}--\eqref{eq:ProgAxE} are deleted in case that the $\gamma_F^\pm$ mention other NLE fluents $F'$, but kept if they only mention LE fluents.
An upper bound is hence given by the sum of sizes of \eqref{eq:ProgAxC}--\eqref{eq:ProgAxE} over {\em all} NLE fluents $F$.
Observe that the total size of occurrences of $\gamma_F^\pm$, $F(\vec{x},S_\alpha)$ atoms, and boolean connectives is less than the size of $\D_{ss}^F[\alpha,S_0]$ when using the representation of Prop. 6,
summing up to less than $\size{\D_{ss}[\alpha,S_0]}$ over all NLE fluents.
This leaves two occurrences of $\NWSC_F(\vec{x})$ and two occurrences of $\SNC_F(\vec{x})$, each of size $\leq w_F$, summing up to
\begin{align*}
S = 4 \sum_{F\in\NLEF(\alpha)} w_F.
\end{align*}
To determine the sum of all $w_F$, we use the intuition that the linear 
recurrence in \eqref{eq:SizeRecurrence} can be unrolled along all source-to-node
paths.
Intuitively, each node starts with a $w_F$ ``weight'' of $w$.
For each child, a node $F$ adds $w_F + 4k + 6a + 10$ to each of its children.
This contribution of $F$ is added multiple times, namely for each path starting in $F$.
Because the out-degree is bounded by 2, the number of distinct paths starting
in $F$ is $\leq 2^{d+1}$, where $d$ is the depth of the dependency graph.
If the number of NLE fluents is $l$, we hence get $S \leq 4 \cdot 2^{d+1} \cdot l \cdot (w + 4k + 6a + 10)$.
The claim follows from the facts that $a$ is a constant, $l \cdot w \in O(\size{\D_{S_0}})$, and $l \cdot k \in O(\size{\D_{ss}})$.
\end{proof}

\subsection*{Proof of Theorem~\ref{thm:fo2nrclosed}}

To prove Theorem~\ref{thm:fo2nrclosed}, we need the following lemma.

\begin{lemma}
\label{lemma:fotwobatnr}
    If $T$ is an $\FOtwo$ theory and semi-definitional wrt $P$, then $\forget(T,P)$ can be expressed in $\FOtwo$.
\end{lemma}
The proof is directly by Thm.~5. $\forget(T,P)$ $\Leftrightarrow$ $ T' \land \band_{\psi \in \mathrm{WSC}_P, \phi \in \mathrm{SNC}_P} \all
\vec{x}.\psi(\vec{x}) \supset \phi(\vec{x})$. Since $T'$, $\psi$, and $\phi$ can mention at most two variables $x,y$, $\forget(T,P) \in \FOtwo$.

\begin{theoremrestated}[Restatement of Theorem~\ref{thm:fo2nrclosed}]
Given a BAT $\D$ that is BAT-NR wrt ground action $\alpha$, if $\D$ is an \FOtwo-BAT wrt $\alpha$, then the progression of $\D_0$ wrt $\D,\alpha$ can be expressed in $\FOtwo$.
\end{theoremrestated}

\begin{proof}
 The result is a consequence of Lemma~\ref{lemma:fotwobatnr}. Since for normal actions, their FO progression is computed by applying Thm.~5 when forgetting non-local-effect fluents, the FO progression is also within $\FOtwo$ if we apply Lemma~\ref{lemma:fotwobatnr}. Besides, for local-effect fluents,  forgetting them via Eq.~(4) leads to a theory that is still in $\FOtwo$ by Thm.~17; this completes the proof.
\end{proof}

\subsection*{Proof of Theorem~\ref{thm:fo2acclosed}}

\begin{theoremrestated}[Restatement of Theorem~\ref{thm:fo2acclosed}]
Given a BAT $\D$ that is BAT-AC wrt ground action $\alpha$, if $\D$ is an $\FOtwo$-BAT wrt $\alpha$, then the progression of $\D_0$ wrt $\D,\alpha$ can be expressed in $\FOtwo$.
\end{theoremrestated}

\begin{proof}
    The proof is pretty much like the proof of Thm.~13 except that we can now use Lemma 19 to show that the formulas generated when forgetting non-local-effect fluents can be rewritten to a $\FOtwo$ formula that is semi-definitional wrt another non-local-effect fluent. This process can be iterated over all non-local-effect fluents. For local-effect fluents, one can use Thm. 17 to ensure that forgetting them leads to an $\FOtwo$ theory as well.
\end{proof}

\subsection*{Proof of Theorem~\ref{thm:utcnrclosed}}

To prove Theorem~\ref{thm:utcnrclosed}, we need the following lemma.

\begin{lemma}
\label{lemma:utcbatnr}
    If $T$ is an $\UTC$ theory and semi-definitional wrt $P$, then $\forget(T,P)$ can be expressed in $\UTC$.
\end{lemma}
The proof is directly from Thm.~5. $\forget(T,P)$ $\Leftrightarrow$ $ T' \land \band_{\psi \in \mathrm{WSC}_P, \phi \in \mathrm{SNC}_P} \all \vec{x}.\psi(\vec{x}) \supset \phi(\vec{x})$. Since $T' \in \UTC$ and $\psi,\phi$ are quantifier-free, $\forget(T,P) \in \UTC$.

\begin{theoremrestated}[Restatement of Theorem~\ref{thm:utcnrclosed}]
Given a BAT $\D$ that is BAT-NR wrt ground action $\alpha$, if $\D$ is an $\UTC$-BAT wrt $\alpha$, then the progression of $\D_0$ wrt $\D,\alpha$ can be expressed in $\UTC$.
\end{theoremrestated}

\begin{proof}
It is a consequence of Lemma~\ref{lemma:utcbatnr}. Since for normal actions, their FO progression is computed by applying Thm.~5 when forgetting non-local-effect fluents, the FO progression is also within $\UTC$ if we apply Lemma~\ref{lemma:utcbatnr}. Besides, for local-effect fluents, forgetting them by Eq.~(4) leads to a theory that is still in $\UTC$ by Thm.~22; this completes the proof.
\end{proof}
}

\section*{Acknowledgements}
Daxin is funded by the National Natural Science Foundation of China (NSFC No. 62506152) and the Fundamental and Interdisciplinary Disciplines Breakthrough Plan of the Ministry of Education of China (No. JYB2025XDXM118).

\bibliographystyle{named}
\bibliography{references}

\end{document}